\documentclass{article}

\usepackage{arxiv}

\usepackage[utf8]{inputenc} 
\usepackage[T1]{fontenc}    
\usepackage{hyperref}       
\usepackage{url}            
\usepackage{booktabs}       
\usepackage{amsfonts}       
\usepackage{nicefrac}       
\usepackage{microtype}      
\usepackage{lipsum}		
\usepackage{graphicx}
\usepackage{natbib}
\usepackage{doi}
\usepackage{amsmath}
\usepackage[caption=false,font=footnotesize,labelfont=rm,textfont=rm]{subfig}

\usepackage{amsthm}
\newtheorem{lemma}{Lemma}
\newtheorem{theorem}{Theorem}

\newtheorem{definition}{Definition}
\newtheorem{assumption}{Assumption}

\title{Rethinking Plasticity in Deep Reinforcement Learning}

\author{%
  Zhiqiang He \\
  The University of Electro-Communications \\
  Tokyo, Japan \\
  \texttt{hezhiqiang@ieee.org} \\
}

\begin{document}
\maketitle

\begin{abstract}
	This paper investigates the fundamental mechanisms driving plasticity loss in deep reinforcement learning (RL), a critical challenge where neural networks lose their ability to adapt to non-stationary environments. While existing research often relies on descriptive metrics like dormant neurons or effective rank, these summaries fail to explain the underlying optimization dynamics. We propose the Optimization-Centric Plasticity (OCP) hypothesis, which posits that plasticity loss arises because optimal points from previous tasks become poor local optima for new tasks, trapping parameters during task transitions and hindering subsequent learning. We theoretically establish the equivalence between neuron dormancy and zero-gradient states, demonstrating that the absence of gradient signals is the primary driver of dormancy. Our experiments reveal that plasticity loss is highly task-specific; notably, networks with high dormancy rates in one task can achieve performance parity with randomly initialized networks when switched to a significantly different task, suggesting that the network's capacity remains intact but is inhibited by the specific optimization landscape. Furthermore, our hypothesis elucidates why parameter constraints mitigate plasticity loss by preventing deep entrenchment in local optima. Validated across diverse non-stationary scenarios, our findings provide a rigorous optimization-based framework for understanding and restoring network plasticity in complex RL domains.
\end{abstract}
\section{Introduction}

The primary challenge in modern artificial intelligence lies in the inherent mismatch between the stationary assumptions of current learning paradigms and the non-stationary nature of real-world environments \cite{dohare2024loss,galashov2024nonstationary}. While artificial neural networks (ANNs) have achieved remarkable success in various domains, their design and training methods are based on the assumption of a stationary data distribution, which greatly limits their practical applicability. This limitation becomes even more apparent when compared to biological neural systems, which possess the ability to continuously adapt through synaptic plasticity \cite{puderbaugh2023neuroplasticity}. This contrast has prompted researchers to investigate the fundamental limitations of ANNs \cite{cohen2022recent}. In particular, the phenomenon of plasticity loss is a critical issue in reinforcement learning (RL) due to its complex dynamics compared to supervised learning \cite{the_primacy_bias}. This challenge is further compounded by the need for RL to adapt value functions to changing rewards and evolving policies, which significantly affect state-action values \cite{elsayed2024weight}. Unlike supervised learning, RL requires a higher level of network plasticity to handle dynamic environments and support strategy evolution during interactive learning \cite{dohare2024loss}. Therefore, understanding and addressing plasticity loss is crucial in RL domains.


Several hypotheses have been proposed to explain the loss of plasticity in neural networks. However, these hypotheses primarily focus on establishing correlations between specific metrics and performance, such as dormant neurons \cite{sokar2023dormant}, loss landscape characteristics \cite{LyleZNPPD23}, and effective rank \cite{gulcehre2022an}. While these metrics are merely a summary of certain phenomena, they do not fully explain the fundamental mechanisms driving plasticity loss \cite{LyleZNPPD23, gulcehre2022an, disentanglingcausesplasticityloss, lewandowski2024directionscurvatureexplanationloss, Hare_Tortoise}. We conducted a simple experiment to highlight the limitations of metric-based analysis. Specifically, we trained the efficient reinforcement learning algorithm PPO under two configurations: one using the ReLU activation function and the other without any activation function. The performance in terms of return and dormancy rate is presented in Figure \ref{figure1}. Interestingly, we observed that a higher dormancy ratio corresponds to a faster convergence rate. This finding contradicts the conclusion that dormancy leads to a loss of plasticity, yet aligns with the performance improvements attributed to sparsity \cite{cheng2024survey}.

\begin{figure}[!h]
\centering
\includegraphics[width=6.5in]{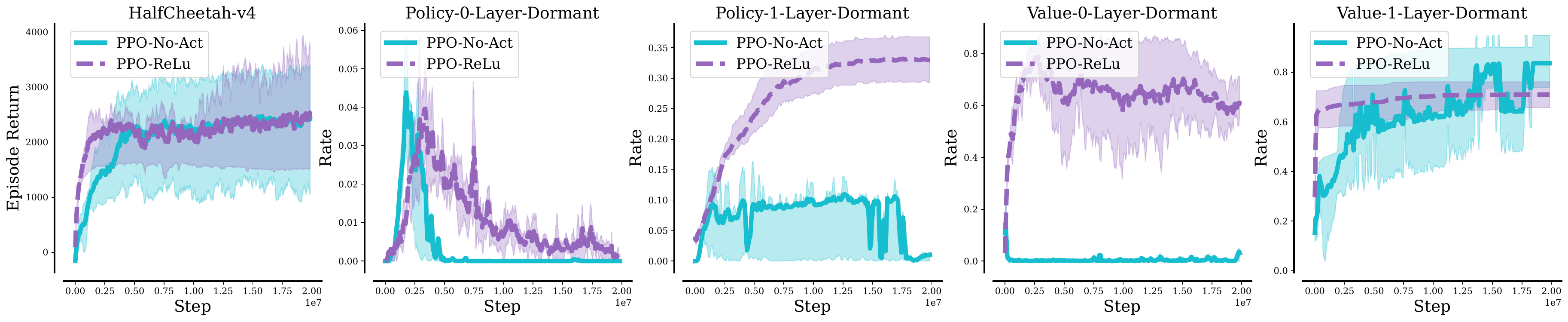}%
\caption{The relationship between return and dormant neurons rate. The leftmost graph shows the return curve, while the next four depict the dormancy rates of policy and value layers. PPO-No-Act lacks activation functions, while PPO-ReLu uses ReLu. PPO-ReLu achieves faster convergence with higher dormancy rates, matching the performance of PPO-No-Act.}
\label{figure1}
\end{figure}


This gap in understanding highlights the critical need for a more rigorous investigation into the underlying mechanisms that govern plasticity loss in neural networks. In contrast to existing metrics, we propose a novel hypothesis that examines plasticity loss from an optimization perspective. Our hypothesis analyzes why this problem occurs during learning and suggests that plasticity loss happens because optimal points from previous tasks become poor optimal points for new tasks. This causes parameters to become trapped in local optima during task transitions, preventing effective learning. Our hypothesis explains several observed phenomena: when tasks change, parameters stuck in local optima from previous tasks negatively impact current task performance, leading networks to adjust other neurons to zero out these trapped neurons' outputs - resulting in neuron dormancy \cite{sokar2023dormant}. Additionally, optimizing other parameters while abandoning those in local optima makes the loss landscape steeper \cite{lyle2022understanding}. This also explains why parameter constraints can increase plasticity \cite{elsayed2024weight,kumar2023maintaining}, as they limit network expressivity from the start, resulting in fewer neurons falling into local optima.


The paper presents a hypothesis and validates it through several aspects: \textbf{Firstly}, previous research has shown that plasticity loss is not a data issue, but rather a result of the neural network's inability to learn from the data \cite{the_primacy_bias}. However, our hypothesis suggests that the core issue lies in the relationship between learning tasks, specifically the relationship between objective functions. When the objective function of the previous task has little correlation with the current task's objective function, it reduces the likelihood of parameters getting trapped in local optima, thus mitigating the plasticity loss problem. \textbf{Secondly}, when tasks change, there is an increase in zero-gradient points, indicating more local optima. Once a neuron enters a zero-gradient state, it is highly likely to remain in the zero-gradient state during subsequent updates. Interestingly, perturbing these local optimal points does not lead to performance degradation, and may even improve performance.

Our main contributions can be summarized as follows:

\begin{enumerate}
    \item \textbf{Novel Hypothesis on Plasticity Loss:} This paper introduces a new hypothesis analyzing plasticity loss from an optimization perspective. It claims that plasticity loss arises because optimal points for previous tasks become suboptimal for new tasks, trapping parameters in local optima during task transitions and hindering neural network learning.
    
    \item \textbf{Explanation of Observed Phenomena:} The hypothesis explains several observed phenomena, including neuron dormancy during task changes, steeper loss landscapes when optimizing parameters near local optima, and the effectiveness of parameter constraints in enhancing plasticity by limiting the network's expressivity in its early stages.
    
    \item \textbf{Experimental Validation and Results:} Validated the hypothesis and techniques through experiments, showing increased local optima during task transitions, performance improvements via perturbation of these optima, and significant reduction of plasticity loss with gradient-free optimization methods. Our results demonstrate substantial improvements in both adaptation speed and final performance across diverse non-stationary scenarios.
\end{enumerate}


\section{Related Work}

Neural networks often exhibit a phenomenon known as diminishing adaptability to new task during the learning process. This has been described in various ways in reinforcement learning literature, such as primacy bias \cite{the_primacy_bias}, dormant neuron phenomenon \cite{sokar2023dormant}, implicit underparameterization \cite{kumar2021implicit}, and capacity loss \cite{lyle2022understanding}. In this paper, we will use the term "loss of plasticity" \cite{dohare2024loss} to refer this phenomenon. While it shares some similarities with catastrophic forgetting \cite{elsayed2024addressing}, the focus of plasticity loss is on the network's ability to learn new patterns, regardless of whether old knowledge is retained or not \cite{dohare2024loss}. In this section, we will review existing literature to analyze various hypotheses proposed to explain this phenomenon and explore current solution approaches.

\textbf{Hypothesis for Plasticity Loss.} Several hypotheses have been proposed to explain the loss of plasticity in neural networks. One prominent view suggests that networks become overly focused on early interaction data \cite{the_primacy_bias}, potentially due to ReLU activation functions causing neurons to become dormant \cite{sokar2023dormant}. However, in cases of redundant features, it may be reasonable for networks to suppress irrelevant features by adjusting weights and output to zero. Other research suggests that imposing specific constraints on network weights can enhance plasticity \cite{elsayed2024weight, kumar2023maintaining, lewandowski2024learning}. Additionally, some studies attribute plasticity loss to properties of the loss landscape \cite{LyleZNPPD23} and effective rank \cite{gulcehre2022an}. Plasticity loss caused by a reduction of Hessian Rank \cite{lewandowski2024directionscurvatureexplanationloss}. Furthermore, \cite{shin2024dash} argue that neural networks tend to memorize noise during the pretraining stage.  Despite these various hypotheses, the fundamental mechanisms underlying plasticity loss remain unclear \cite{LyleZNPPD23, gulcehre2022an,disentanglingcausesplasticityloss,lewandowski2024directionscurvatureexplanationloss, Hare_Tortoise}. In contrast to existing work, we propose a hypothesis from an optimization perspective that can explain various observed metrics such as neuron dormancy and loss landscape characteristics.

\textbf{Preventing plasticity loss.} Research on neural plasticity in deep reinforcement learning is still in its early stages, with current studies falling into two main categories. The first category focuses on preventing the loss of plasticity during the training process through various implementation-level techniques. These approaches include: normalization methods that regulate input or output distributions \cite{lyle2024normalization}; weight constraint mechanisms that prevent excessive parameter magnitudes \cite{elsayed2024weight} or parameters that deviate too far from the initial distribution \cite{kumar2023maintaining} or early feature predictions \cite{lyle2022understanding}. Additionally, new activation functions have been designed to ensure that partial neurons are always activated for the same feature \cite{crelu}, and adaptive batch size strategies have been implemented \cite{ceron2023small}. It has been observed that large weight magnitudes can negatively impact optimization dynamics and are often associated with overfitting, leading to decreased performance and potentially contributing to the challenges encountered during learning \cite{elsayed2024weight}. Weight Clip limits the change in function, which is crucial for on-policy reinforcement learning to prevent the policy from deviating too much from its previous state \cite{elsayed2024weight}.

\textbf{Solution for plasticity loss.} The second category focuses on restoring plasticity in networks where neurons have already lost their adaptive capabilities, known as the parameter resetting method \cite{dohare2024loss, juliani2024study, liu2024neuroplastic}. These approaches include several innovative methods: Neuroplastic Expansion \cite{liu2024neuroplastic}, which enables incremental growth of network neurons; Continual Backpropagation \cite{dohare2024loss}, which identifies ineffective neurons and reinitializes their parameters; Continual backpropagation sets a new utility metric and lets a small fraction of less-used units are continually and randomly reinitialized \cite{dohare2024loss}. ReDO \cite{sokar2023dormant}, which detects and reinitializes dormant neurons; and Plasticity Injection \cite{nikishin2024deep}, which maintains network adaptability by introducing additional neurons during the training process. These approaches commonly rely on a binary metric that classifies neurons as either active or dormant based solely on whether their outputs are non-zero or zero. However, this dormancy metric lacks theoretical justification, as zero outputs may represent optimal and meaningful neural responses in certain contexts. In fact, some studies even adopt this metric as an optimization objective \cite{xu2024drm}.Some papers focus on optimizers, resetting them to clear historical information for enhanced network plasticity, particularly during rapid shifts in data distribution \cite{asadi2023resetting}. Unlike these methods, we propose a gradient-based metric to quantify neuronal plasticity, aiming to provide a more accurate understanding of network plasticity characteristics.

\section{Optimization-Centric Plasticity (OCP) Hypothesis}

\subsection{An Intuitive Example of OCP}

To demonstrate the OCP hypothesis, let's imagine a situation where we have two tasks, $f_{1}$ and $f_{2}$, that need to be completed in sequence. As shown in Figure \ref{example_ocp}, the transition from \textbf{Task 1} to \textbf{Task 2} brings about a significant change in the objective function. While the neural network may successfully reach an optimal solution (indicated by the red point) for \textbf{Task 1}, this solution may not be the best for \textbf{Task 2} and even harm its performance. In other words, the network parameters at the red point are at risk of getting stuck in a local optimum that negatively affects \textbf{Task 2}. When training proceeds with data from task $f_{2}$, a subset of network parameters are already confined to these local optima. Consequently, the neural network must compensate by making more dramatic adjustments to the remaining, untrapped parameters to accommodate task $f_{2}$. This constrained optimization leads to two significant effects: a steeper loss landscape due to the reduced dimensionality of adjustable parameters \cite{LyleZNPPD23}, and the network potentially learning to suppress outputs associated with the trapped parameters if they interfere with task $f_{2}$ performance. This mechanism explains the observed variable increases in dormant rates across different scenarios \cite{sokar2023dormant}, as the phenomenon's occurrence depends heavily on the specific task relationships and their respective optimization landscapes. Furthermore, this framework elucidates why regularization methods that constrain the parameter space can effectively mitigate plasticity loss: by preventing parameters from becoming too deeply entrenched in local optima for any single task, they preserve the network's adaptability for subsequent learning.

\begin{figure*}[!htbp]
\centering
\includegraphics[width=4.5in]{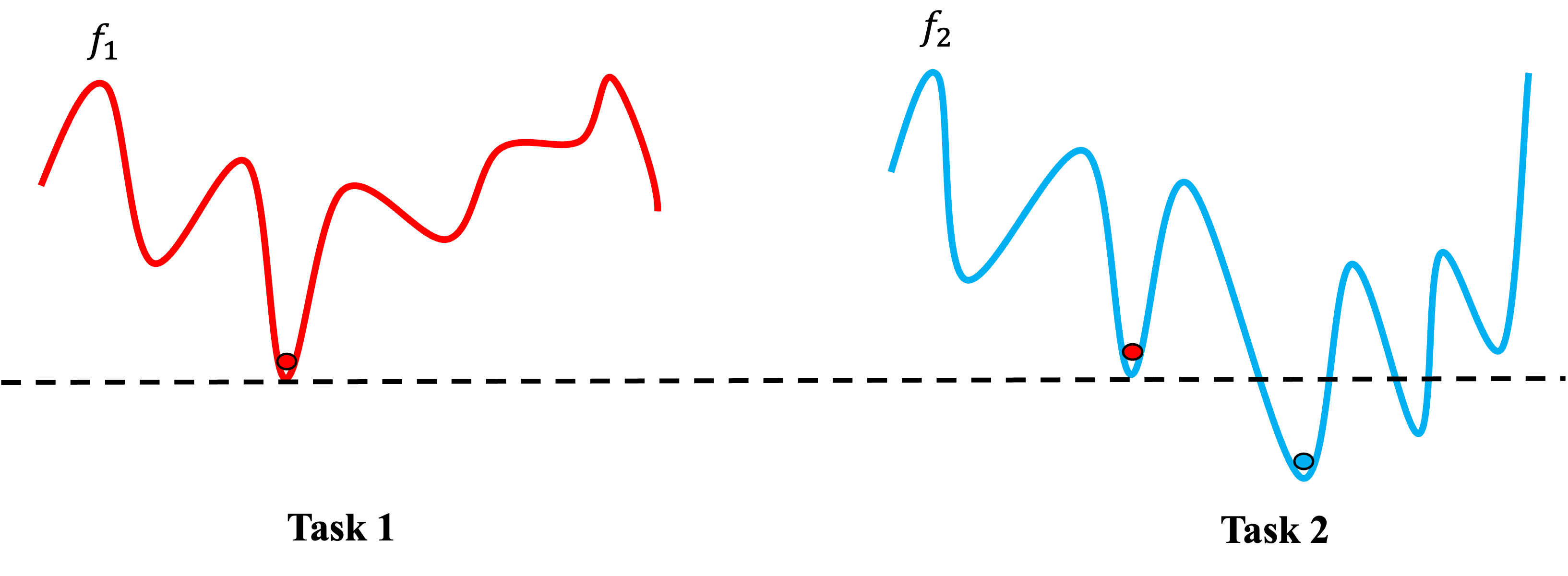}
\caption{Illustration of the OCP hypothesis.}
\label{example_ocp}
\end{figure*}

\subsection{Task Relevance and Plasticity Loss}

In this subsection, our goal is to demonstrate that the loss of neural network plasticity is dependent on the task at hand, thus proving that the phenomenon is caused by being stuck in a local optimum. Specifically, the inability of a neural network to continue learning in a given task, as indicated by a high proportion of dormant neurons, does not necessarily mean that the network is incapable of learning. Instead, it suggests that the network struggles with tasks that are similar to the one being performed. However, when the task changes significantly, such as when the optimization landscape undergoes a substantial shift, the network's neurons are able to learn and perform similarly to randomly initialized neurons, without experiencing a loss of plasticity.

\begin{align}
    y = 2.5 X_0 - 1.2 X_1^2 + 0.8 \sin(X_2) + 1.5 \cos(X_3) + 0.7 X_4 X_5 - 0.3 X_6^3 + e^{-0.1 X_7^2} + 1.1 X_8 - 0.5 X_9^2 + \nonumber \\ 
     0.9 \tanh(X_{10}) + 0.2 X_{11}^2 - 0.6 \sqrt{|X_{12}|} + 0.5 X_{13} X_{14} - 0.4 X_{15} + 0.3 X_{16} + \varepsilon, \varepsilon \sim \mathcal{N}(0, \sigma^2).
    \label{regression_problem}
\end{align}

To test our hypothesis on the loss of task-specific plasticity, we utilized a pre-trained value network from PPO (Proximal Policy Optimization) for a regression task. This task involves a 17-dimensional input vector $\mathbf{X}=(X_{0}, X_{1}, \cdots, X_{16}) \in \mathbb{R}^{17}$ and a single output value $y \in \mathbb{R}$. The true data-generating process is represented by Equation \ref{regression_problem}. More information on the experimental parameters, comparisons, and the reasoning behind the chosen functional form can be found in Appendix \ref{task_relevance_appendix}.

\begin{figure}[!h]
\centering
\includegraphics[width=2.5in]{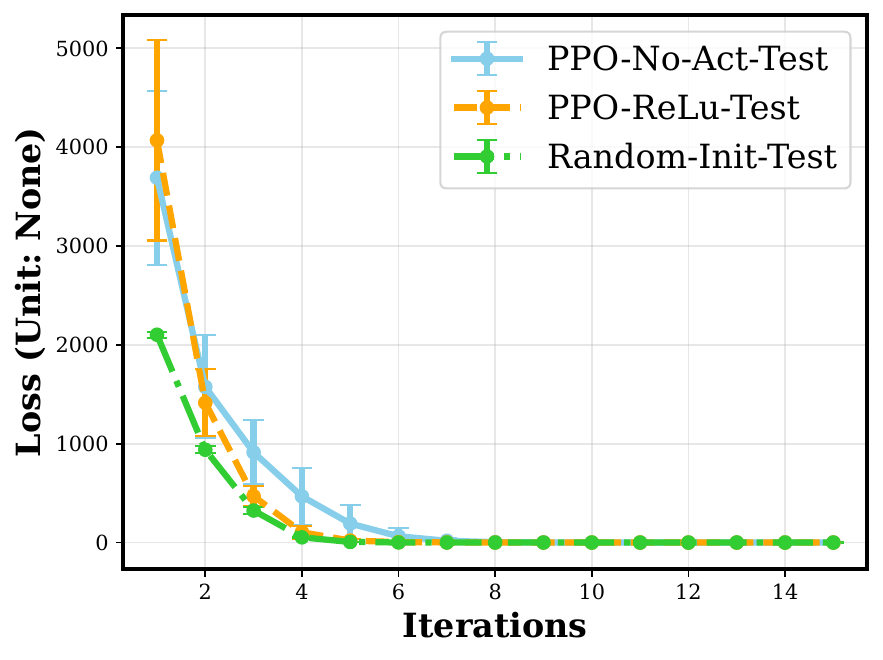}
\includegraphics[width=2.5in]{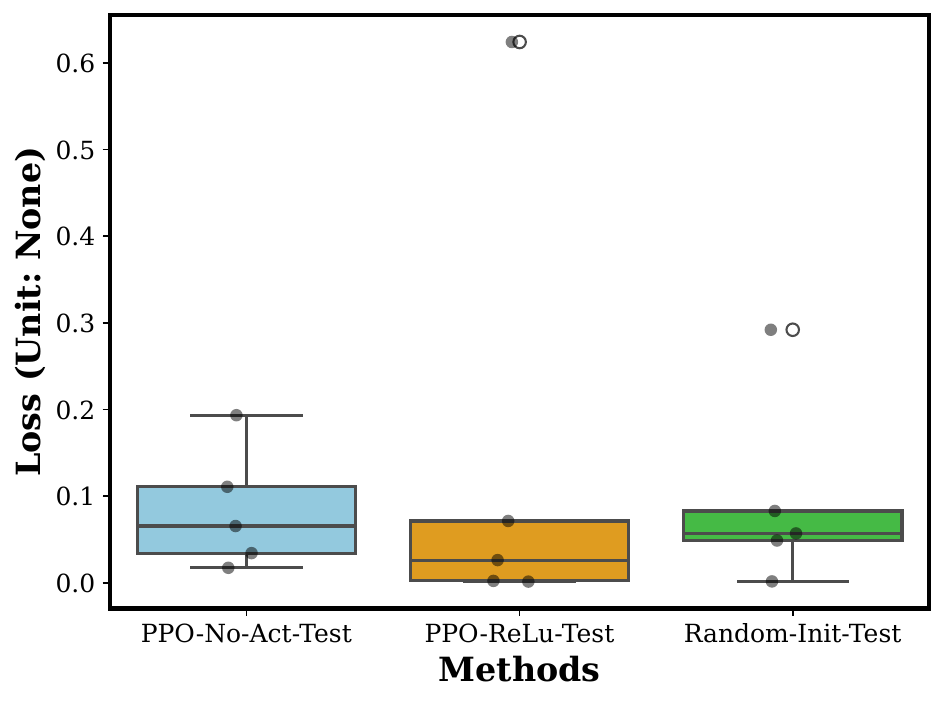}
\label{task_change}
\caption{The convergence curves of the loss function and the final performance comparison on the test set were observed under task switching with a learning rate of 0.01.}
\label{example}
\end{figure}

As shown in Figure 3, networks trained for 20 million steps with a low proportion of dormant neurons (PPO-No-Act-Test), networks trained for 20 million steps with a high proportion of dormant neurons (PPO-ReLU-Test), and randomly initialized networks (Random-Init-Test) all exhibited similar learning performance on the regression task. This suggests that when the task changes significantly, previously ineffective networks can successfully learn the new task, indicating that plasticity loss is not a fundamental issue. These results further validate our hypothesis that the loss of neural network plasticity is a task-specific phenomenon caused by neurons becoming trapped in a local optimum specific to the task. This evidence strongly supports our hypothesis.


\subsection{Dormancy as a Gradient-Induced Phenomenon}

In this section, we will establish the theoretical equivalence between zero-gradient neurons and dormancy. We will demonstrate that the absence of gradient signals is what drives neurons into a dormant state. Additionally, we will provide an optimization-based explanation for the tendency of the proportion of dormant neurons to increase when the underlying task changes.


We consider a neural network layer $l$ with input space $\mathcal{X}_l \subseteq \mathbb{R}^{n_l}$, parameterized by $\mathbf{W}_l \in \mathbb{R}^{n_{l+1} \times n_l}$ and $\mathbf{b}_l \in \mathbb{R}^{n_{l+1}}$, and with an activation function $\sigma_l: \mathbb{R} \to \mathbb{R}$ applied elementwise. The forward pass of layer $l$ is $f_l(\mathbf{x}) = \sigma_l(\mathbf{W}_l \mathbf{x} + \mathbf{b}_{l})$, and the output of the $i$-th neuron in layer $l$ is $h_{l, i}(\mathbf{x})=\sigma_l\left(\mathbf{w}_{l, i}^T \mathbf{x}+b_{l, i}\right)$, where $\mathbf{w}_{l,i}^T$ is the $i$-th row of $\mathbf{W}_l$.

\begin{definition}[{\textbf{Dormant Neuron}}]
Let $H_{l}$ be the number of neurons in layer $l$. The dormancy index $s_{l,i}$ for neuron ($l,i$) is defined as:

\begin{equation}
    s_{l,i} = \frac{\mathbb{E}_{\mathbf{x}\in D}|h_{l,i}(\mathbf{x})|}{\frac{1}{H_{l}}\sum_{k\in h}\mathbb{E}_{\mathbf{x}\in D}|h_{l, k}(\mathbf{x})|}
\end{equation}

where $\mathbb{E}_{\mathbf{x} \in D}$ denotes the expectation over the input $\mathbf{x}$ drawn from distribution $D$.
\end{definition}

\begin{assumption}[\textbf{Continuity and Differentiability}]
\label{continuity}
$h_{l,i}(\mathbf{x})$ is continuously differentiable with respect to $\mathbf{x}$, i.e., $h_{l,i} \in C^{1}(\mathbb{R}^{n_l})$.
\end{assumption}

\begin{assumption}[\textbf{Boundedness}]
\label{boundedness}
$\exists M > 0$ such that for all $k \in h$:
$\mathbb{E}_{\mathbf{\mathbf{x}} \in D}|h_{l,i}(\mathbf{\mathbf{x}})| \leq M$.
\end{assumption}

\begin{assumption}[\textbf{Non-degeneracy}]
\label{non_degeneracy}
$\exists m > 0$ such that $\frac{1}{H_{l}}\sum_{k\in h}\mathbb{E}_{\mathbf{x}\in D}|h_{l,i}(\mathbf{x})| \geq m$. This ensures the denominator in the dormancy index definition is strictly positive.
\end{assumption}

\begin{lemma}[\textbf{Zero Output Lemma}]
    \label{ZeroOutLemma}
    If $s_{l,i} = 0$, then $h_{l,i}(\mathbf{x}) = 0$ for all $\mathbf{x} \in D$. 
\end{lemma}
\begin{proof}
    Since $s_{l,i} = 0$,
    \begin{equation}
        0=s_{l, i}=\frac{\mathbb{E}_{\mathbf{x} \in D}\left|h_{l, i}(\mathbf{x})\right|}{\frac{1}{H_l} \sum_{k=1}^{H_l} \mathbb{E}_{\mathbf{x} \in D}\left|h_{l, k}(\mathbf{x})\right|}.
    \end{equation}

    By the non-degeneracy assumption, the denominator is at least $m>0$. Thus, 
    \begin{equation}
        \mathbb{E}_{\mathbf{x} \in D}\left|h_{l, i}(\mathbf{x})\right|=0.
    \end{equation}
    Since $\left|h_{l, i}(\mathbf{x})\right| \geq 0$, the expectation of a nonnegative random variable is zero if and only if the variable is zero almost everywhere. Therefore,
    \begin{equation}
        h_{l, i}(\mathbf{x})=0 \quad \text {for almost all } \mathbf{x} \in D .
    \end{equation}
    By the continuity of $h_{l, i}$, if it were nonzero at any point $\mathbf{x}_{0} \in D$, continuity would imply a neighborhood around $\mathbf{x}_{0}$ where $h_{l, i}$ remains nonzero, contradicting the almost-everywhere zero condition. Hence, $h_{l,i}(\mathbf{x})=0$ for every $\mathbf{x} \in D$.
\end{proof}

\begin{lemma}[\textbf{Dormancy-to-Gradient Lemma}]
    \label{Dormancy_to_Gradient_Lemma}
    If $s_{l,i} = 0$, then $\Delta h_{l,i}(\mathbf{x}) = 0$ for all $\mathbf{x} \in D$.
\end{lemma}
\begin{proof}
    From Lemma \ref{ZeroOutLemma}, we know that if $s_{l,i}=0$, $\forall \mathbf{x} \in D, h_{l, i}(\mathbf{x})=0$. Let $\mathbf{x} \in D$ be arbitrary, and consider any perturbation $\delta_{x} \in \mathbb{R}^{n_l}$. By the Mean Value Theorem for vector-valued functions, there exists $t \in [0, 1]$ such that
    \begin{equation}
        h_{l, i}\left(\mathbf{x}+\delta_{\mathbf{x}}\right)-h_{l, i}(\mathbf{x})=\nabla h_{l, i}\left(\mathbf{x}+t \delta_{\mathbf{x}}\right)^T \delta_{\mathbf{x}}.
    \end{equation}
    Since $h_{l,i}(\mathbf{x}) = 0$ for every $\mathbf{x} \in D$, this difference is zero:
    \begin{equation}
        0=\nabla h_{l, i}\left(\mathbf{x}+t \delta_{\mathbf{x}}\right)^T \delta_{\mathbf{x}}.
    \end{equation}
    This equality must hold for all $\mathbf{x} \in D$ and for every possible $\delta_{x} \in \mathbb{R}^{n_l}$.  Suppose, for contradiction, that there exists some point $\mathbf{y} \in D$ where $\nabla h_{l, i}(\mathbf{y}) \neq 0$. Take $\delta_{\mathbf{x}}=\epsilon \nabla h_{l, i}(\mathbf{y})$ for a small $\epsilon>0$, then
    \begin{equation}
        \nabla h_{l, i}\left(\mathbf{y}+t \epsilon \nabla h_{l, i}(\mathbf{y})\right)^T\left(\epsilon \nabla h_{l, i}(\mathbf{y})\right)=0.
    \end{equation}
    As $\epsilon \rightarrow 0$, continuity of $\nabla h_{l, i}$ ensures $\nabla h_{l, i}\left(\mathbf{y}+t \epsilon \nabla h_{l, i}(\mathbf{y})\right) \rightarrow \nabla h_{l, i}(\mathbf{y})$, which is nonzero by assumption. Hence the inner product $\nabla h_{l, i}(\mathbf{y})^T\left(\epsilon \nabla h_{l, i}(\mathbf{y})\right)$ would not vanish for small $\epsilon \neq 0$, yielding a contradiction. Thus, no such $\mathbf{y}$ can exist, and we conclude
    \begin{equation}
        \nabla h_{l, i}(\mathbf{x})=0 \quad \forall \mathbf{x} \in D.
    \end{equation}
\end{proof}

\begin{lemma}[\textbf{Gradient-to-Dormancy Lemma}]
    Suppose $\nabla h_{l, i}(\mathbf{x})=0$ for all $\mathbf{x} \in \mathcal{X}_l$. Then $h_{l, i}(\mathbf{x})$ is constant on $\mathcal{X}_{l}$. Moreover, if $h_{l,i}(\mathbf{x}) = 0$ at any point $\mathbf{x}_0 \in \mathcal{X}_l$, then $h_{l, i}(\mathbf{x})=0$ for all $\mathbf{x} \in \mathcal{X}_l$, and hence the neuron is dormant $\left(s_{l, i}=0\right)$.
\end{lemma}
\begin{proof}
    Since $\nabla h_{l, i}(\mathbf{x})=0$ for every $\mathbf{x} \in \mathcal{X}_{l}$. Then $h_{l,i}$ does not change with respect to $\mathbf{x}$. By standard multivariable calculus, any function whose gradient is zero everywhere on a connected domain must be constant. Concretely, for any two points $\mathbf{x}_{1}, \mathbf{x}_{2} \in \mathcal{X}_{l}$, we have
    \begin{equation}
        h_{l, i}\left(\mathbf{x}_2\right)-h_{l, i}\left(\mathbf{x}_1\right)=\int_{\mathbf{x}_1}^{\mathbf{x}_2} \nabla h_{l, i}(\mathbf{u}) d \mathbf{u}=0,
    \end{equation}
    indicating $h_{l,i}(\mathbf{x}_{2}) = h_{l,i}(\mathbf{x}_{1})$. Therefore, $h_{l,i}(\mathbf{x})$ is constant on $\mathcal{X}_{l}$. Denote this constant by $c$, so $h_{l,i}(\mathbf{x}) = c$ for all $\mathbf{x} \in \mathcal{X}_{l}$.
    
    Next, assume there exists a point $\mathbf{x}_{0} \in \mathcal{X}$ such that $h_{l,i}(\mathbf{x}_{0})=0$, then $c=0$. Hence $h_{l,i}(\textbf{x})=0$ for every $\mathbf{x} \in \mathcal{X}_{l}$. In particular, for $\mathbf{x} \in D \subseteq \mathcal{X}_l$, we have $h_{l,i}(\mathbf{x}) = 0$. Consequently,
    \begin{equation}
        \mathbb{E}_{\mathbf{x} \in D}\left|h_{l, i}(\mathbf{x})\right|=0 \quad \Longrightarrow \quad s_{l, i}=\frac{0}{\frac{1}{H_H} \sum_k \mathbb{E}_{\mathbf{x} \in D}\left|h_{l, k}(\mathbf{x})\right|}=0 .
    \end{equation}
    Thus the neuron is fully dormant.
    
    \textbf{Remark.} If $c \neq 0$, then $\mathbb{E}_{\mathbf{x} \in D}\left|h_{l, i}(\mathbf{x})\right|$ would be $|c| \neq 0$, so $s_{l, i}$ could potentially be nonzero, indicating the neuron is not dormant. Therefore, the key condition that $h_{l,i}(\mathbf{x})=0$ at some point $\mathbf{x}_{0}$ in $\mathbf{X}_{l}$ (thus forcing $c=0$) is crucial to concluding $s_{l,i}=0$. In many network architectures, one can ensure $h_{l, i}(0)=0$ by design (e.g., bias initialized to zero and activation is ReLU), or rely on data/architectural constraints that force a zero constant rather than a nonzero one.
\end{proof}

\begin{theorem}[\textbf{Equivalence of Dormancy and Zero Gradient}]
\label{EquivalenceofDormancyandZeroGradient}
    Let $D \subseteq \mathcal{X}_l \subseteq \mathbb{R}^{n_l}$ be the domain from which inputs $\mathbf{x}$ are drawn. Consider layer $l$ with $H_{l}$ neurons. For the $i$-th neuron in layer $l$, let
    \begin{equation}
        h_{l, i}(\mathbf{x})=\sigma_l\left(\mathbf{w}_{l, i}^T \mathbf{x}+b_{l, i}\right),
    \end{equation}
    where $\sigma_{l}$ is continuously differentiable and applied elementwise. The neuron’s \emph{dormancy index} is defined as:
    \begin{equation}
        s_{l, i}=\frac{\mathbb{E}_{\mathbf{x} \in D}\left|h_{l, i}(\mathbf{x})\right|}{\frac{1}{H_l} \sum_{k=1}^{H_l} \mathbb{E}_{\mathbf{x} \in D}\left|h_{l, k}(\mathbf{x})\right|}
    \end{equation}
    \textbf{Assume}:
    \begin{enumerate}
        \item \textbf{(Continuity and Differentiability)} $h_{l,i}(\mathbf{x})$ is in $C^{1}(\mathbb{R}^{n_l})$.
        \item \textbf{(Boundedness)} $\exists M>0$: $\mathbb{E}_{\mathbf{x} \in D}\left|h_{l, i}(\mathbf{x})\right| \leq M$.
        \item \textbf{(Non-degeneracy)} $\exists m>0: \frac{1}{H_l} \sum_{k=1}^{H_l} \mathbb{E}_{\mathbf{x} \in D}\left|h_{l, k}(\mathbf{x})\right| \geq m$.
    \end{enumerate}
    Then the following two statements are equivalent:
    \begin{itemize}
        \item (A) \textbf{Dormancy}: $s_{l,i}=0$. Equivalently, $\mathbb{E}_{\mathbf{x} \in D}\left|h_{l, i}(\mathbf{x})\right|=0$.
        \item (B) \textbf{Zero Gradient (on $D$) plus at least one zero activation}:
        \begin{equation}
            \nabla h_{l, i}(\mathbf{x})=0 \quad \forall \mathbf{x} \in D \quad \text { and } \quad \exists \mathbf{x}_0 \in D: h_{l, i}\left(\mathbf{x}_0\right)=0.
        \end{equation}
    \end{itemize}
\end{theorem}
\begin{proof} \textbf{(A) $\Rightarrow$ (B)}: Suppose $s_{l,i} = 0$. By Lemma \ref{ZeroOutLemma}, $h_{l,i}(\mathbf{x}) = 0$ for all $\mathbf{x} \in D$. Then by the Lemma \ref{Dormancy_to_Gradient_Lemma}, we get $\nabla h_{l, i}(\mathbf{x})=0$ for all $\mathbf{x} \in D$. Clearly, in this scenario, pick $\mathbf{x}_{0} \in D$ arbitrarily; $h_{l, i}(\mathbf{x}_{0}) = 0$. Thus the right-hand side of the equivalence is satisfied. \textbf{(A) $\Leftarrow$ (B)}: Conversely, assume $\nabla h_{l, i}(\mathbf{x})=0$ for all $\mathbf{x} \in D$ and there exists at least one point $\mathbf{x}_{0} \in D$ where $h_{l,i}(\mathbf{x}_{0}) = 0$. Since $\nabla h_{l, i} \equiv 0$ on $D$, the neuron's output $h_{l,i}$ is constant throughout $D$. That constant must be zero (as it is zero at $\mathbf{x_{0}}$). Hence $h_{l,i}(\mathbf{x}) = 0$ for all $\mathbf{x} \in D$. By Lemma \ref{ZeroOutLemma} (applied in reverse logic), if the neuron is identically zero on $D$, then $s_{l, i}=0$.
\end{proof}

\begin{figure}[!h]
\centering
\includegraphics[width=6.5in]{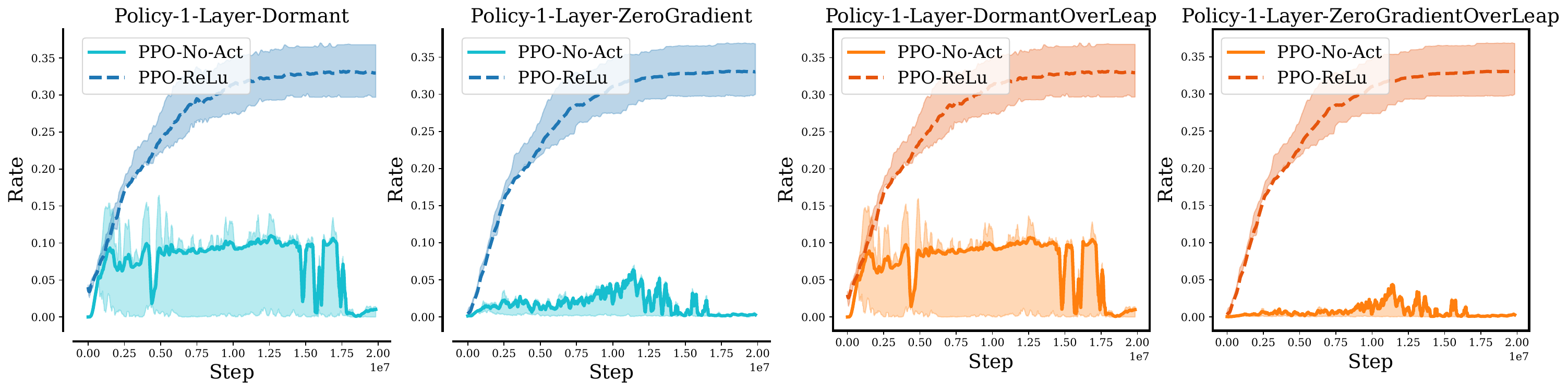}
\includegraphics[width=6.5in]{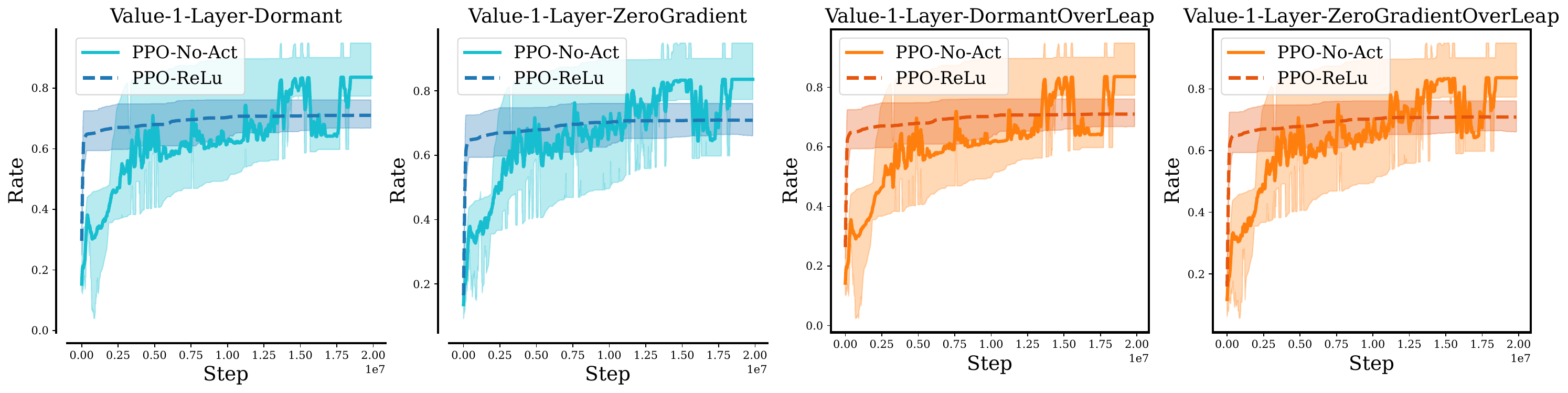}
\caption{The relationship between dormant neurons and zero gradient neurons in last layer of policy and value neural network in HalfCheetah-v4 environment. The Dormancy Index and ZeroGradient Index show strong correlation. The OverLeap metric reflects the proportion of neurons that remain dormant or zero-gradient across iterations, further highlighting this consistency.}
\label{relationship_policy_value}
\end{figure}

\textbf{Theorem \ref{EquivalenceofDormancyandZeroGradient}} states that a neuron with a dormancy index of zero will result in zero outputs on the dataset, causing a zero gradient that will keep it in a dormant state. On the other hand, if a neuron's gradient is negligible, it will either remain or become close to zero in output, resulting in no updates. Even small fluctuations that push its parameters to a constant or near-zero state will further reinforce this condition, particularly in networks with ReLU activation functions. The experimental results in Figure \ref{relationship_policy_value} also validate our theory: there is a strong correlation between dormant neurons and neurons with zero gradients. Similar to dormant neurons, Zero Gradient neurons gradually increase during the training process. Once a neuron enters a dormant or zero-gradient state, it becomes irrecoverable. The OverLeap plot illustrates this phenomenon clearly. Further experimental setups and analyses can be found in Appendix \ref{dormantneuron_and_zerogradientneuron}.

From an optimization standpoint, parameter updates follow gradient signals. However, when the task changes, the loss function also changes, often resulting in parameters being near old local optima with negligible or zero gradients for the new objective. This can lead to neurons specialized in outdated features becoming or remaining dormant. Even a slight shift in the task can cause the network's parameters to stay close to their previous minima, perpetuating low gradients and increasing the rate of dormancy. On the other hand, a significant task change usually invalidates the old minima, resulting in stronger gradients that can "wake up" previously dormant neurons. Once a neuron's output reaches zero, its gradient also becomes small, creating a self-reinforcing loop that keeps it dormant unless strong new signals are received. Theoretically and practically, dormancy is closely linked to having zero or vanishingly small gradients.

\bibliographystyle{unsrtnat}
\bibliography{references}  

\newpage

\appendix
\section{Supplementary Details on Task Relevance and Plasticity Loss}
\label{task_relevance_appendix}

The state space of HalfCheetah-v4 is 17-dimensional. In order to assess the plasticity of its Value function, we have devised a regression problem with a 17-dimensional input vector $\mathbf{X}=(X_{0}, X_{1}, \cdots, X_{16}) \in \mathbb{R}^{17}$ and a scalar output $y \in \mathbb{R}$. The true underlying data-generating process is described by Equation \ref{regression_problem_appendix}.

\begin{align}
    y = 2.5 X_0 - 1.2 X_1^2 + 0.8 \sin(X_2) + 1.5 \cos(X_3) + 0.7 X_4 X_5 - 0.3 X_6^3 + e^{-0.1 X_7^2} + 1.1 X_8 - 0.5 X_9^2 + \nonumber \\ 
     0.9 \tanh(X_{10}) + 0.2 X_{11}^2 - 0.6 \sqrt{|X_{12}|} + 0.5 X_{13} X_{14} - 0.4 X_{15} + 0.3 X_{16} + \varepsilon, \varepsilon \sim \mathcal{N}(0, \sigma^2).
    \label{regression_problem_appendix}
\end{align}

In Equation \ref{regression_problem_appendix}, the variable $\varepsilon$ represents an additive noise term that is randomly sampled from a zero-mean Gaussian distribution with a variance of $\sigma^{2}$, where $\sigma=0.1$. The function $y$ includes a variety of mathematical operators, such as linear, quadratic, cubic, trigonometric (sin, cos), exponential ($e^{-0.1 X_{7}^{2}}$), and hyperbolic (tanh) transformations, as well as interaction terms (e.g., $X_{4}X_{5}$, $X_{13}X_{14}$) and non-smooth functions (e.g., $\sqrt{|X_{12}|}$). This diverse structure introduces significant nonlinearity and complexity to the regression task. The main purpose of this design is to ensure that the problem has a wide range of optimization landscapes, allowing us to test whether a neural network can effectively adapt to a new and substantially different task if it becomes trapped in local optima under one task.

\begin{figure}[ht]
\centering
\subfloat[]{
    \includegraphics[width=2.5in]{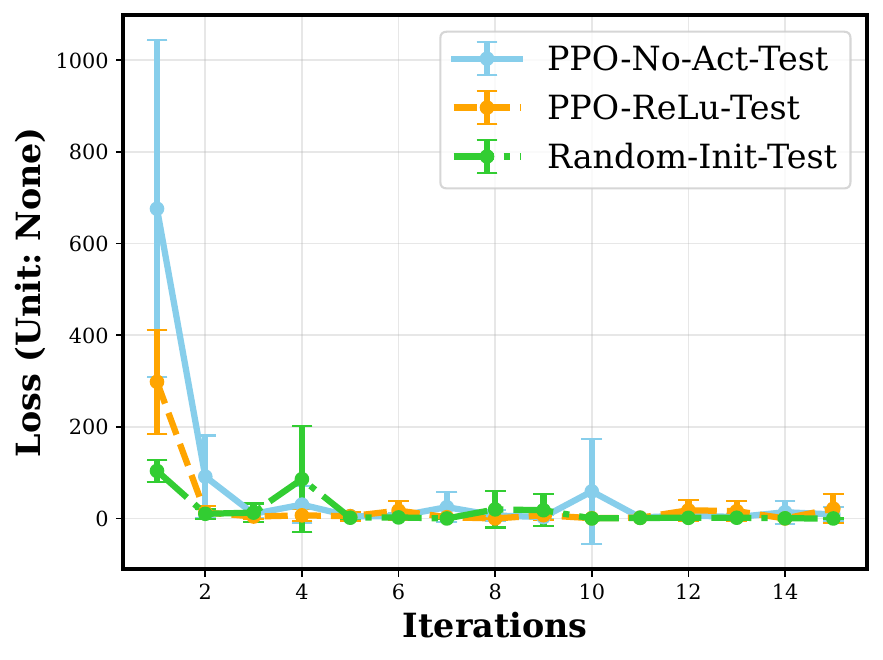}
    \includegraphics[width=2.5in]{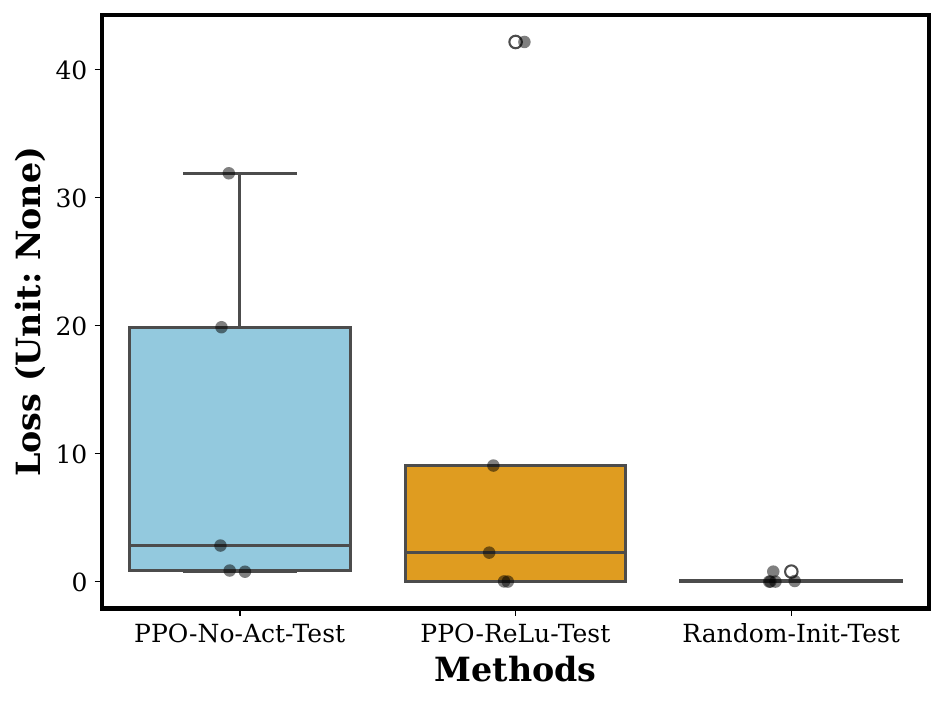}
\label{fig0_a}}%
\vfill
\subfloat[]{
    \includegraphics[width=2.5in]{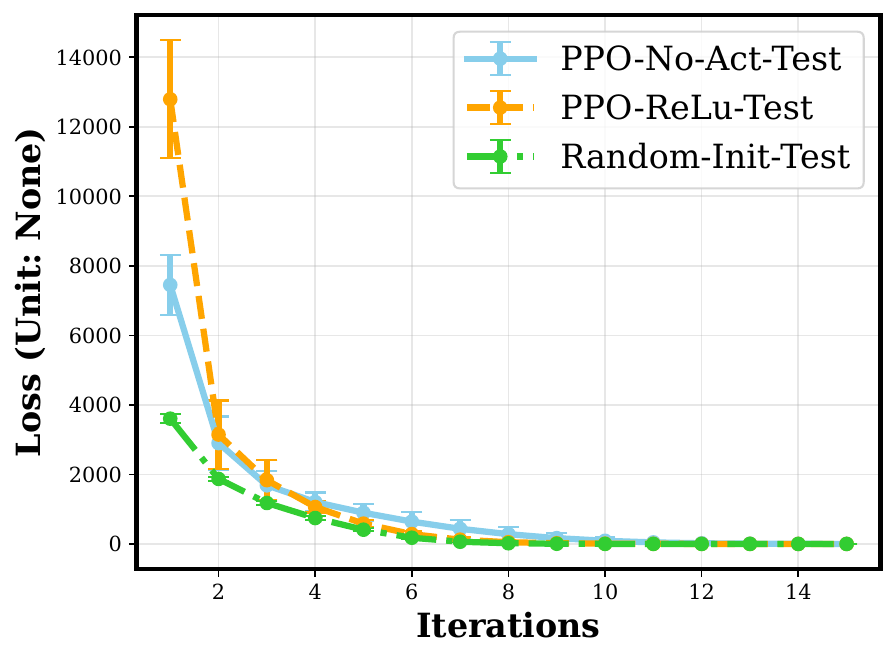}
    \includegraphics[width=2.5in]{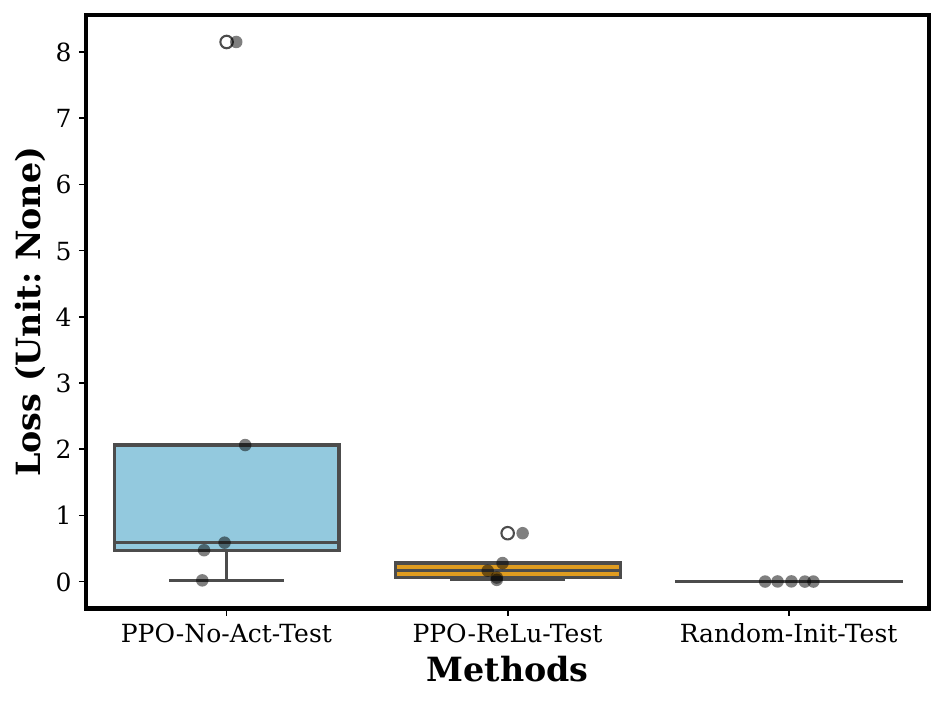}
\label{fig0_b}}%

\caption{Loss convergence curves and final convergence results under different learning rates. (a) shows the results for a learning rate of 0.01, while (b) shows the results for a learning rate of 0.005.}
\label{diverse_lr_rate}
\end{figure}

The training and test datasets consist of 1,000 randomly generated samples, and the loss function used is Mean Squared Error (MSE). The experimental results are shown in Figure \ref{diverse_lr_rate}. "PPO-No-Act-Test" refers to the results on the test dataset where the initial weights of the value function were obtained from the PPO algorithm after 20 million steps in the HalfCheetah-v4 environment without activation functions. "PPO-ReLU-Test" represents the test results using initial weights from the PPO algorithm after 20 million steps in the same environment with ReLu activation functions. "Random-Init" indicates the test results with randomly initialized weights. It can be observed that under different learning rates, the convergence speed and final results of pre-trained neural networks with a high proportion of dormant weights show no significant difference compared to those of randomly initialized networks. This demonstrates that the plasticity of neural networks is highly task-dependent.

\newpage
\section{The Experimental Setups And Analyses of Dormant Neuron and Zero Gradient Neuron}
\label{dormantneuron_and_zerogradientneuron}

\begin{definition}[{\textbf{Mean Absolute Gradient Intensity, MAGI}}]
We define the Mean Absolute Gradient Intensity (MAGI) for the $i$-th neuron in layer $l$ as the mean absolute value of the gradients of all incoming weights to that neuron, averaged over the input feature dimension:

\begin{equation}
    G_i^{(l)}=\frac{1}{n_{\text {in }}} \sum_{j=1}^{n_{\text {in }}}\left|\frac{\partial S}{\partial w_{i j}^{(l)}}\right|
\end{equation}

where $n_{in}$ is the input dimension of layer $l$, representing the number of input features or neurons connected to the $i$-th neuron in the previous layer. $w_{i,j}^{(l)}$ is the weight connecting the $j$-th input feature (or neuron) to the $i$-th neuron in layer $l$. $S$ is the target function, defined as:

\begin{equation}
    S=\sum_{k=1}^{n_{\text {batch }}} \sum_{m=1}^{n_{\text {out }}} \mathbf{Y}_{k, m}
\end{equation}

where $n_{batch}$ is the batch size, representing the number of samples in the current batch. $n_{out}$ is the output dimension of layer $l$, i.e., the number of neurons in the current layer. $\mathbf{Y} \in \mathbb{R}^{n_{\text {batch }} \times n_{\text {out }}}$ is the output of layer $l$, computed as:

\begin{equation}
    \mathbf{Y}=\mathbf{X} \cdot \mathbf{W}_l^{\top}+\mathbf{b}_l
\end{equation}

where $\mathbf{X} \in \mathbb{R}^{n_{\text {batch }} \times n_{\text {in }}}$ is the input to layer $l$, consisting of $n_{batch}$ samples, each with $n_{in}$ features. $\mathbf{W}_l \in \mathbb{R}^{n_{\text {out }} \times n_{\text {in }}}$ is the weight matrix of layer $l$, connecting the input to the output neurons. $\mathbf{b}_l \in \mathbb{R}^{n_{\text {out }}}$ is the bias vector of layer $l$.
\end{definition}

The gradient $\frac{\partial S}{\partial w_{i j}^{(l)}}$ represents the sensitivity of the target function $S$ to changes in the weight $w_{i j}^{(l)}$, and the absolute value is taken to avoid positive and negative gradients canceling each other out. The mean over $n_{in}$ input feature reflects the overall sensitivity of the neuron to its incoming connections. MAGI quantifies the importance of a neuron in layer $l$ by summarizing the average gradient magnitude across all its incoming weights. It provides a concise and interpretable measure of the role a neuron plays during optimization and can be used to identify important neurons, prune redundant ones, and guide architectural decisions in model design.

\begin{definition}[\textbf{Overlap Coefficient for Neuron Dormancy or Zero Gradient}]

The overlap coefficient quantifies the similarity between the set of dormant or zero gradient neurons in each layer of a neural network at the current iteration and last iteration \cite{sokar2023dormant}. Let $A$ denote the set of dormant or zero gradient neurons at the current iteration, and $B$ denote the previous time step dormant and zero gradient. The overlap coefficient between $A$ and $B$ is defined as :

\begin{equation}
    \operatorname{overlap}(A, B)=\frac{|A \cap B|}{\min (|A|,|B|)},
\end{equation}

This metric is used to measure the proportion of dormant neurons or zero-gradient neurons that remain consistent over time, providing insights into the network's insight state.
    
\end{definition}

In this study, we have chosen to use the PPO algorithm as our basic algorithm framework. We have trained our model for a total of 20 million timesteps. To optimize the learning process, we have implemented an annealed learning rate starting at  $1 \times 10^{-4}$ and a weight decay of $1 \times 10^{-4}$. In order to stabilize the training, we have also utilized a gamma discount factor of 0.99 and Generalized Advantage Estimation (GAE) with $\lambda = 0.95$. Our policy gradient is optimized using 32 minibatches over 10 update epochs per iteration, with gradient clipping applied using a maximum norm of 0.5. Additionally, we have employed a clipped surrogate loss with a clipping coefficient of 0.2 and ensured value function stability through a value loss coefficient of 0.5. The number of dormant neurons and zero-gradient neurons in each layer of the Policy and Value neural networks is shown in Figure \ref{dormantandzerogradient}.

\begin{figure}[!h]

\centering
\includegraphics[width=3.2in]{figures/policy_1_dormant_and_gradient.pdf}
\includegraphics[width=3.2in]{figures/value_1_dormant_and_gradient.pdf}
\includegraphics[width=3.2in]{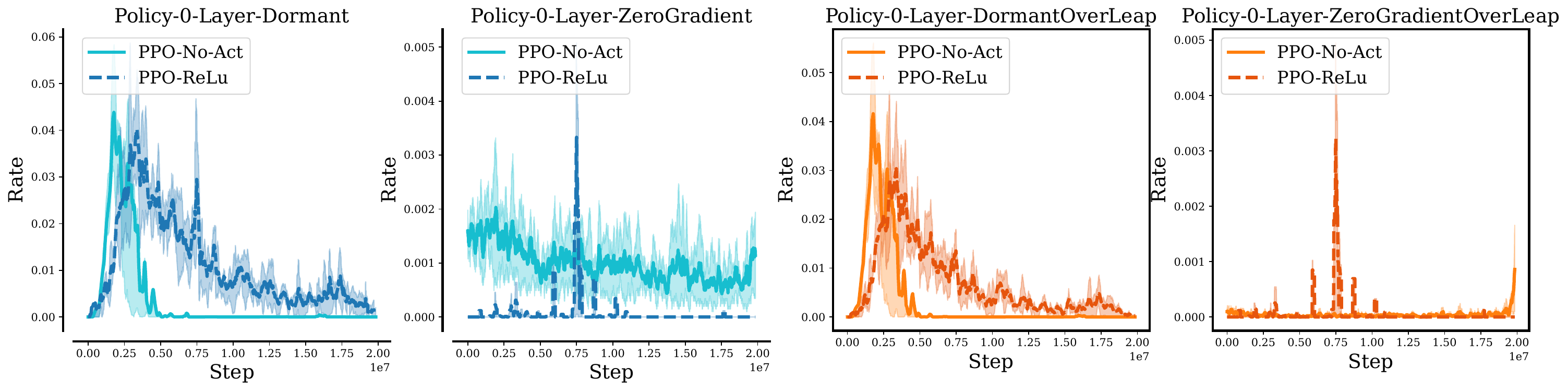}
\includegraphics[width=3.2in]{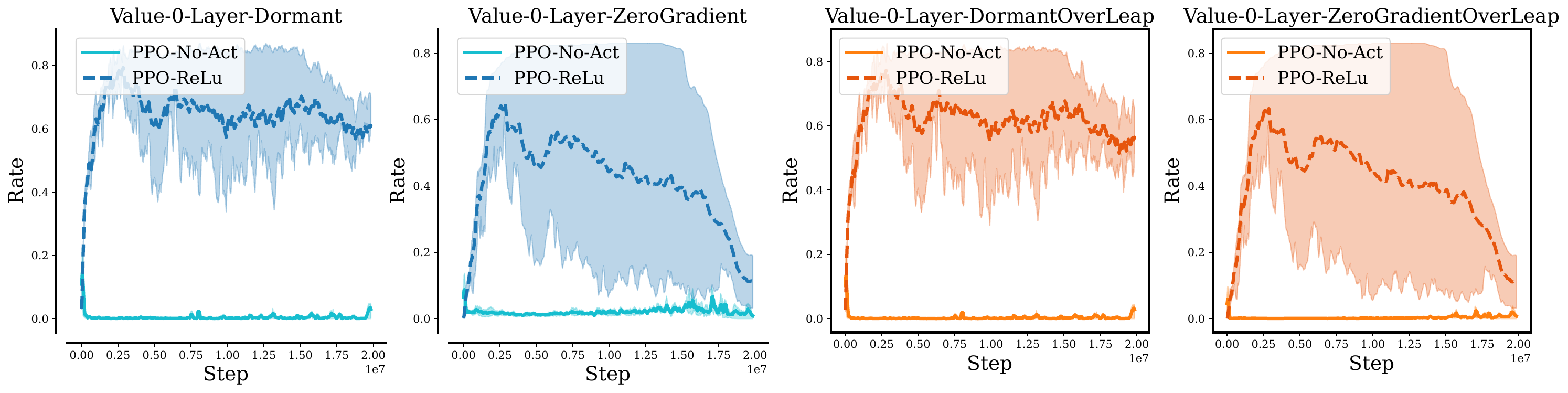}
\caption{The relationship between dormant neurons and zero gradient neurons in each layer neural network in HalfCheetah-v4 environment.}
\label{dormantandzerogradient}
\end{figure}

From the experimental results, we can draw the following conclusions:

\begin{itemize}
    \item Similar to dormant neurons, Zero Gradient neurons gradually increase during the training process.
    \item Once neurons enter the Zero Gradient state, they cannot exit this state.
    \item The proportion of dormant neurons and zero-gradient neurons exhibits a high degree of consistency, particularly in the final layer of the neural network.
    \item The dormancy rate and zero-gradient proportion of neurons in the final layer are significantly higher than the zero-gradient proportion in the preceding layer.
    \item Under the Value-0-Layer, Zero Gradient neurons possess a certain degree of recovery ability and can also help reduce dormant neurons.
\end{itemize}

Overall, the experimental results validate our Theorem \ref{EquivalenceofDormancyandZeroGradient} (\textbf{Equivalence of Dormancy and Zero Gradient}). However, there are still some discrepancies between the results and the theoretical predictions, such as the dormancy rate of neurons being slightly higher than the zero-gradient proportion in the 0th layer of the Value network. Additionally, why are the dormancy rate and zero-gradient proportion higher in the final layer of neurons?

There are several common reasons why one might observe a neuron whose empirical “dormancy” measures out to zero (or very close to zero) yet whose gradient is not exactly zero in practice, despite the theoretical result stating that a strictly zero‐output neuron must have a strictly zero gradient. Below are some explanations:

\textbf{Finite Sampling vs. Theoretical Domain}

The proof assumes the neuron’s output is identically zero for all inputs $\mathbf{x}$ in the data domain. In a real experiment, you only have a finite sample of points (minibatches, training sets, etc.), and the neuron might be zero on average or nearly zero for that finite sample, rather than genuinely zero everywhere. Consequently, even if the dormancy metric is numerically zero, the neuron may still produce small nonzero values for some unobserved inputs, hence giving rise to small but nonzero gradients.

\textbf{Floating-Point and Numerical Approximation}
Computers store and process numbers in floating-point representation. Values that are “close to zero” often end up being rounded to zero (or appear extremely small) during forward passes and in statistics like the dormancy measure. However, backpropagation and gradient calculations may amplify tiny differences, producing nonzero gradients even if the output is extremely close to zero in the forward pass. This discrepancy is a byproduct of finite precision arithmetic and can cause “dormant” neurons to have small, nonzero gradients.

\textbf{Momentum and Other Optimizer Dynamics}
If your optimization includes momentum (e.g., in SGD with momentum, Adam, RMSProp, etc.), the gradient update for a neuron can remain nonzero temporarily, even if the neuron’s instantaneous output is consistently near zero. The optimizer’s internal state (momentum buffer, adaptive learning rates, etc.) may carry over nonzero update terms that slightly adjust the neuron’s parameters despite the neuron having negligible or zero output on the dataset.

\textbf{Regularization and Other Loss Components}
In many practical neural networks, the total loss function is not just the standard training objective (e.g., cross‐entropy) but also includes regularization terms (weight decay, L1/L2 penalties, etc.). Even if the neuron output is zero, changes in that neuron’s parameters might still affect the regularization loss term, thereby giving a small gradient signal with respect to those parameters.

\textbf{“Almost Dormant” vs. Strictly Dormant}
The theoretical result “If output is strictly zero for all $\mathbf{x} \in D$, then gradient is strictly zero” is exact only in an ideal mathematical sense. In practice, a dormancy index of zero typically means “so small that it was rounded or thresholded to zero,” not that the neuron is provably zero on the entire input domain. A minuscule nonzero output can still lead to a nontrivial gradient.

\textbf{Distribution Shift or Data Subset Effects}
Sometimes the neuron is dormant on the training set or a certain subset of data but might not be truly zero on the broader distribution or on another part of the dataset. Therefore, when gradients are computed over a slightly different batch or distribution slice, that neuron might receive a small gradient even though it was “dormant” under the original dataset slice.

There are several common reasons why the final layer in a neural network is often more prone to having dormant (or effectively zero‐gradient) neurons:

\textbf{Direct Dependency on the Loss}
The gradient signals reaching the final layer come directly from the loss function. If certain output neurons (or individual dimensions in the final layer) do not reduce the loss under the current task, they receive negligible or zero gradients. Consequently, those neurons see no incentive to update, causing their outputs to drift toward zero (dormancy) and remain there.

\textbf{Reduced Intermediary “Revival” Opportunities}
In earlier layers, if a neuron is near‐zero, it can still get “woken up” by meaningful signals from subsequent layers or skip connections. By contrast, in the final layer, there are no additional transformations after it—if the loss does not depend on that particular neuron’s output, there is no mechanism to reignite its gradient. Thus, once it goes dormant, it stays that way.

\textbf{Sharper Task Specialization}
The final layer often encodes the most task-specific features (such as class logits). When the task changes (e.g., in transfer learning or domain shift), certain final‐layer neurons that were specific to now‐irrelevant classes or features may suddenly become useless for reducing the new loss. With no gradient feedback, those neurons quickly or permanently drop to zero output.

\textbf{Sparsity Induced by Softmax or Output Activation}
In classification, the final layer frequently feeds into a softmax function. If one or more logits consistently fail to be competitive, they produce negligible probability mass and generate little to no gradient over multiple training steps. This leads to a self‐reinforcing effect: low output yields low gradient, which keeps the output low.

\textbf{Fewer Parameter Interactions}
Earlier layers often have rich parameter interactions (via multiple subsequent layers), giving nonzero gradient signals to neurons that could still help shape intermediate representations. In the final layer, each output neuron is directly tied to a single dimension of the final prediction. If that dimension is not contributing to reducing the loss, its gradient can vanish without being “rescued” by other parts of the network.

\newpage
\section{Statistics about Weights}

\begin{figure}[!h]
\centering
\includegraphics[width=6.5in]{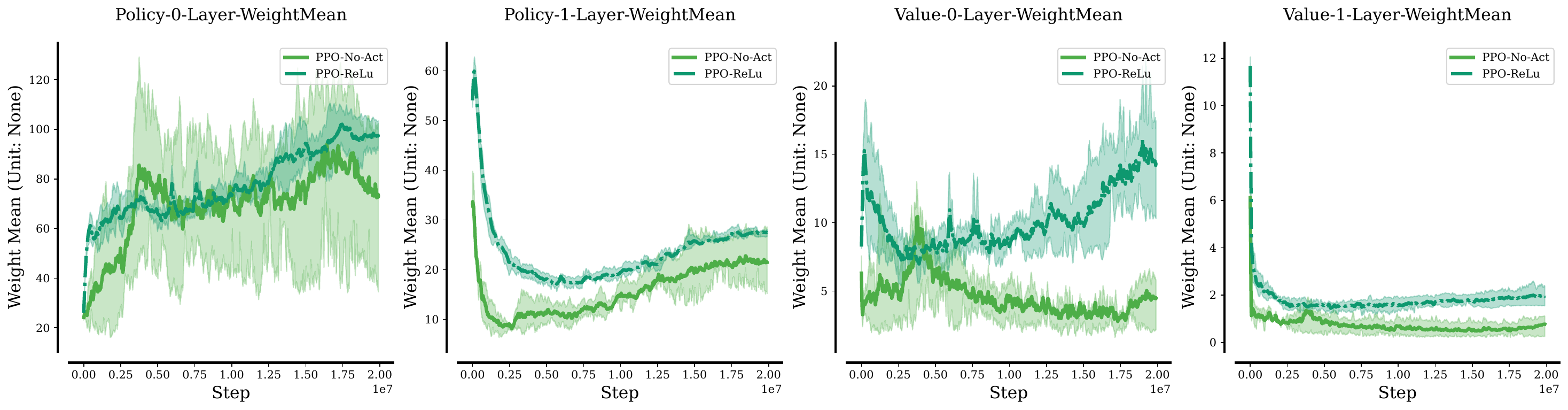}
\includegraphics[width=6.5in]{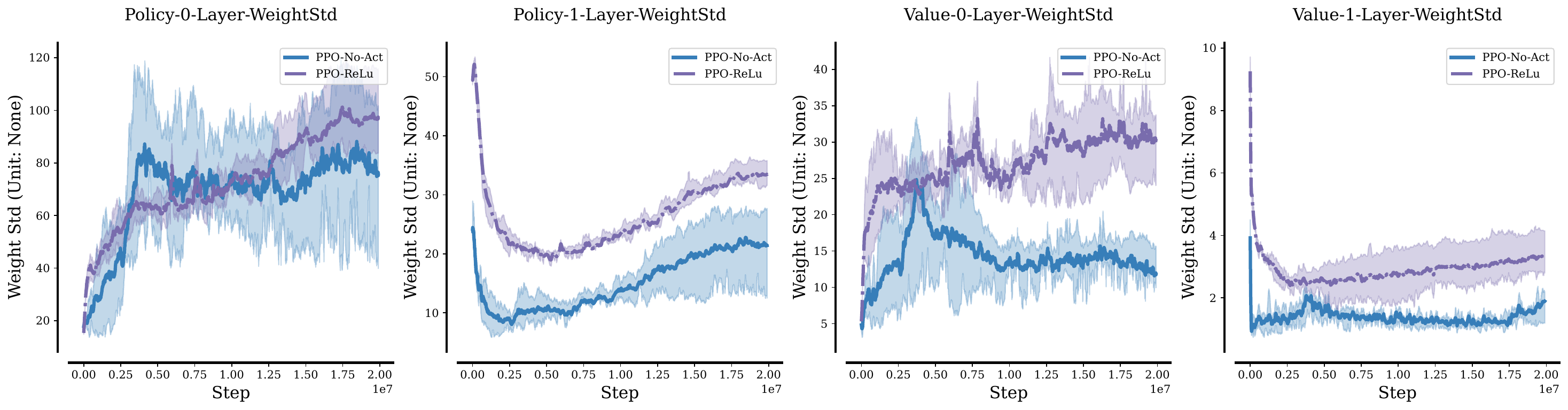}
\includegraphics[width=6.5in]{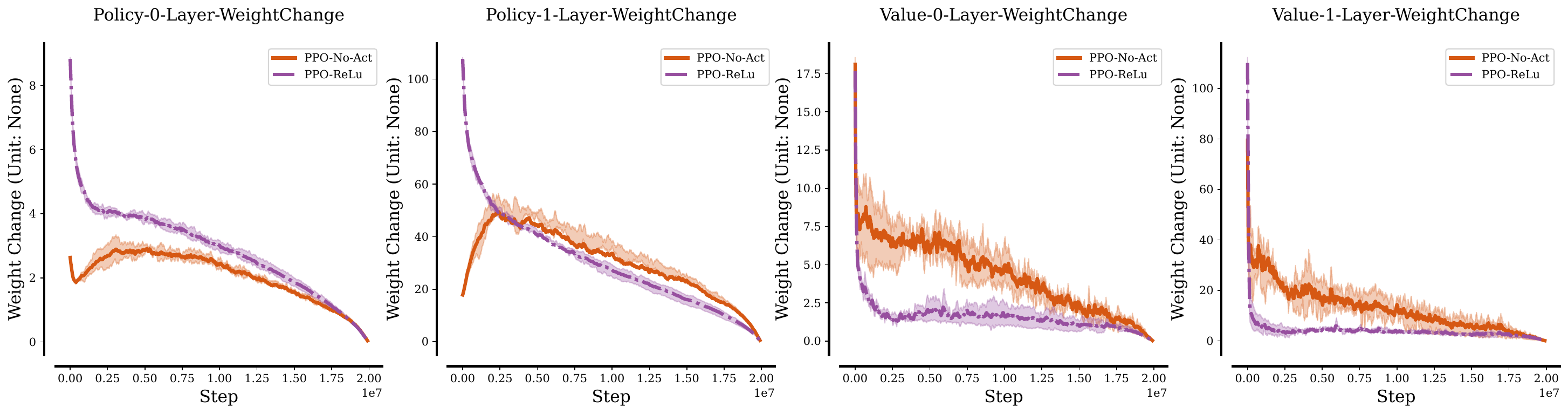}
\label{return_and_dormant}

\caption{Weight Statistic Information}
\label{example}
\end{figure}

\newpage
\section{Statistics about Ranks}

\begin{figure}[!h]
\centering
\includegraphics[width=6.5in]{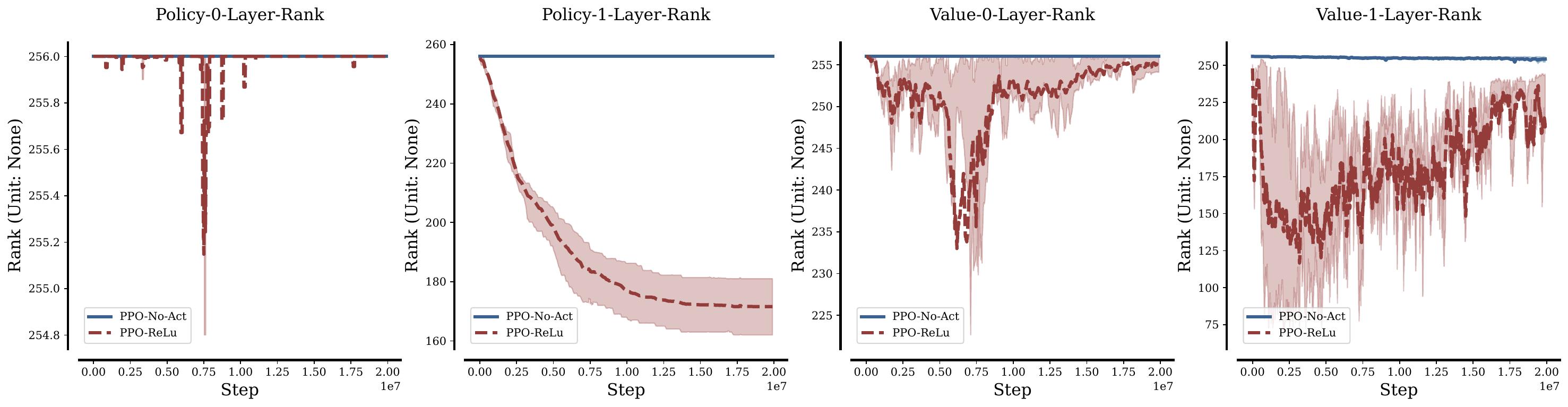}
\includegraphics[width=6.5in]{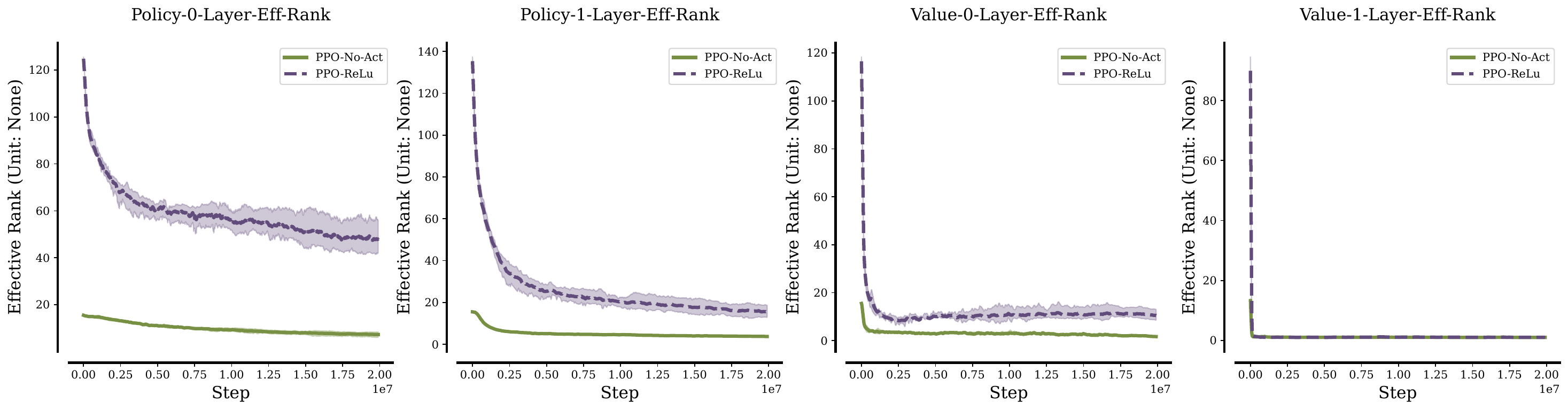}
\includegraphics[width=6.5in]{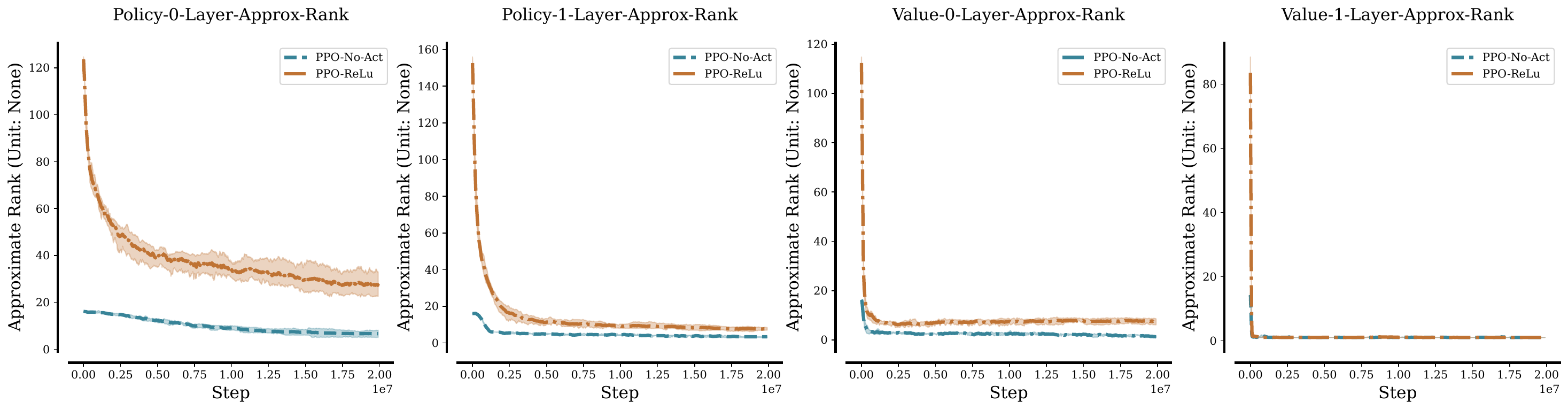}
\includegraphics[width=6.5in]{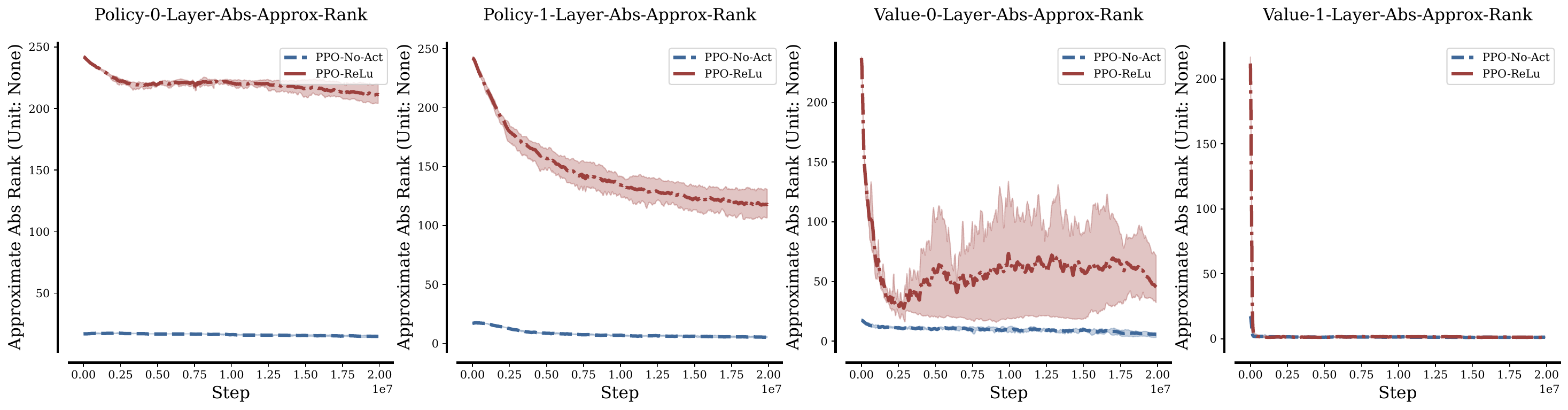}
\label{return_and_dormant}

\caption{Weight Statistic Information}
\label{example}
\end{figure}






\end{document}